\documentclass{article}

\usepackage{amsmath}
\usepackage{amsthm}
\usepackage{amsfonts}
\usepackage{amssymb}
\usepackage{graphicx}
\usepackage{stmaryrd}
\usepackage{float}
\usepackage{slashbox}
\usepackage{caption}
\usepackage{subfigure}
\usepackage[english]{babel} 
\usepackage[utf8]{inputenc}
\usepackage[T1]{fontenc}
\usepackage{tikz}

\newtheorem{theorem}{Theorem}

\newtheorem{definition}[theorem]{Definition}

\begin{document}

\title{On Simulated Annealing Dedicated to Maximin Latin Hypercube Designs}
\author{Pierre Berg\'e, Kaourintin Le Guiban,\\
Arpad Rimmel, Joanna Tomasik, \\
LRI, CentraleSupélec, Université Paris-Saclay\\
Bat 650, Rue Noetzlin, 91405 Orsay, France}

\maketitle

\begin{abstract}
The goal of our research was to enhance local search heuristics used to construct Latin Hypercube Designs. First, we introduce the \textit{1D-move} perturbation to improve the space exploration performed by these algorithms. Second, we propose a new evaluation function $\psi_{p,\sigma}$ specifically targeting the Maximin criterion. 

Exhaustive series of experiments with Simulated Annealing, which we used as a typically well-behaving local search heuristics, confirm that our goal was reached as the result we obtained surpasses the best scores reported in the literature. Furthermore, the $\psi_{p,\sigma}$ function seems very promising for a wide spectrum of optimization problems through the Maximin criterion.
\end{abstract}


\section{Introduction}

The study of complex systems usually requires a considerable computation time. To speed up computations, the system may be replaced by a faster approximating model. To create this model, a set of outcomes for different parameter values is needed. The set of parameter values has an impact on the accuracy of the approximating model. Different sampling methods for this set of parameters has been proposed in~\cite{mckay79}. If we note $k$ the number of inputs of the system and $n$ the number of possible values taken by an input variable $x_{k}$, the choice of $n$ sample vectors can be represented by points on a hypercube of size $n$ and dimension $k$. Among the designs proposed, we focus on the Maximin Latin Hypercube Design (LHD). 

The LHD implements the Latin constraint: each coordinate in $\left[\left|1;n\right|\right]$ must appear only once in every dimension. In other words, the coordinates of any pairs of nodes differ in all dimensions. Moreover, the Maximin constraint means that we search the configuration with the maximal $d_{\min}$, where $d_{\min}$ is the minimal distance between two points of the design.
An instance is defined by the values of the dimension $k$ and the size $n$. Consequently, an instance will be noted $k/n$, for example $10/50$.

There are exactly $\binom{n}{2} = \frac{n(n-1)}{2}$ distances between points. They may be ordered $d_{1} \leq d_{2} \leq ... \leq d_{i} \leq d_{i+1} ... \leq d_{\binom{n}{2}} $, by definition, $d_{\min} = d_{1}$. In the remainder, we often refer to square values: $D_{i} = d_{i}^{2}$.

As the Maximin LHD problem is believed to be NP-hard, heuristics are widely used to solve it. The use of deterministic methods are, for the moment, rather limited: branch-and-bound was only used with $k \leq 3$ by \cite{vandam07,vandam09} (a highscore is a minimal maximal distance between any pair of design points for a given instance) . The survey \cite{rimmel14} of metaheuristics sums up their performance on Maximin LHD and reports that SA not only outperformed other evolutionary algorithms but also improved many of the previous highscores for $k \leq 10$ and $n \leq 25$. For this reason we choose to work on the improvement of SA applied to Maximin LHDs. 

After introducing usual methods used to solve Maximin LHD with SA, we propose a new mutation and a new evaluation function to improve on the best current scores. Eventually, we show how it permits to exceed a great part of the scores presented by the literature. These results should not exclusively be considered in the particular SA context as they may offer the opportunity to boost the performance of local search algorithms for different designs with the Maximin constraint.

{\bf N.B.} {\em This reports completes our paper~\cite{Ber+16} providing the precisions and details which had to be omitted in its camera-ready version due to the page number limit. It should  thus be cited together with~\cite{Ber+16}.} 

\section{Typical approach to solve Maximin LHD with Simulated Annealing}
	
SA is a metaheuristic most commonly used for discrete search spaces inspired by a metallurgical process. It consists in visiting the search space with a perturbation on the configurations and deciding whether the mutated configuration should be selected. That is why it alternates phases of heating and slow cooling to influence this choice: as in physics, by controlling the cooling process, we give an opportunity to the configuration to find the lowest energy. 
	
	It is proven that SA converges to the global minimum with the Metropolis probability \cite{aarts85}. Several ingredients are compulsory for SA \cite{kirkpatrick83}: a perturbation (a mutation), an evaluation function $H_{\mbox{\scriptsize {pot}}}$ (also called potential energy) and a temperature decrease $T(k)$, where $k$ is the iteration number. The acceptance probability depends on $T(k)$ and the gap of potential energy between the new and the old configuration. When the current configuration at the iteration number $k$ is $\omega$, a configuration $\omega'$ will be accepted with the Metropolis probability:
\[
p_k = \min \left(1,e^{\frac{H_{\mbox{\scriptsize {pot}}}(\omega') - H_{\mbox{\scriptsize {pot}}}(\omega)}{KT(k)}}\right)
\]
Constant $K$ is typically fixed to 1. A configuration $\omega$ is defined by $n$ points with $k$ coordinates respecting the Latin constraints.
	
Survey \cite{rimmel14} examined several perturbations among which $m_2$ was the most efficient. It deals with a pair of points: a randomly chosen one and a critical one. A critical point is a point involved in $d_{\min}$. Mutation $m_2$ transposes the coordinates of these points in one dimension. The authors of \cite{rimmel14} proposed mutation $m_3$ which is a variant of $m_2$ as the transposition takes place in the dimension which ensures a better $d_{\min}$ for a subsequent configuration $\omega'$. Mutation $m_3$ outperformed $m_2$ for $9/10$, $4/25$ and $8/20$.
	
	Article \cite{rimmel14} compared two evaluation functions: $-d_{\mbox{\scriptsize {min}}}$ and $\phi_{p}$ introduced in~\cite{morris95}:
	
	\begin{equation}
	\phi_{p} = \left(\sum_{i=1}^{\binom{n}{2}} d_{i}^{-p}\right)^{\frac{1}{p}}.
	\end{equation}
	
	Function $\phi_{p}$ is more efficient than $-d_{\mbox{\scriptsize {min}}}$ certainly because it takes into account changes on every distance whereas the function $-d_{\mbox{\scriptsize {min}}}$ only considers the shortest distance of the configuration. As the paper \cite{rimmel14} obtained the best score of the literature, we used the same values to set the values of parameter~$p$.
	
\section{New perturbation targeting LHD}

\subsection{Principle of the perturbation}
	
We use $m_2$ as a basis to construct a better performing perturbation. To clarify its principle, we define the notion of the neighborhood: 
	
\begin{definition}[Neighbor of a point]
	For a given instance $k/n$, a point $p_{1}$ of a configuration $\omega$ is a neighbor of the point $p_{2}$ if and only if there is a dimension $j$ for which coordinates of these two points are the closest possible. In other terms, $\exists j \in \left[\left|1;k\right|\right] \mbox{such that} \left|p_{1}(j) - p_{2}(j)\right| = 1$.
	
\end{definition}
	
The new mutation \textit{1D--move} consists in taking a critical point as before and taking one of its neighbor. Then we exchange the coordinates in one of the dimensions concerned by the neighborhood.
	
	We choose a $3/5$ instance to illustrate \textit{1D--move}. Table~\ref{mutation} gives the coordinates of the points of a configuration $\omega$. First, we choose a critical point: $d_{\mbox{\scriptsize {min}}}$ is determined by points $p_{1}$ and $p_{2}$, so we take $p_{1}$. Points $p_{2}$ (on axis $x$, $y$ and $z$), $p_{3}$ (on axis $y$) and $p_{4}$ (on axis $z$) are neighbors of $p_{1}$. For the sake of the example, we shall choose $p_{4}$ as a neighbor. Then, we exchange the coordinates of $p_{1}$ and $p_{4}$ on axis $z$ because $p_{1}$ and $p_{4}$ are neighbors through dimension $z$. The new configuration is also given in Table~\ref{mutation}. 
		
		\begin{table*}[t]
	\centering
	$
		\begin{array}{|c|c|c|c|c|c|}
		\hline
		\mbox{Points} & p_1 & p_2 & p_3 & p_4 & p_5 \\
		\hline
		x & 0 & 1 & 2 & 3 & 4\\
		\hline
		y & 1 & 2 & 0 & 4 & 3\\
		\hline
		z & \textbf{2} & 1 & 4 & \textbf{3} & 0\\
		\hline
		\end{array}
	$
	\hspace*{5mm}
	$
		\begin{array}{|c|c|c|c|c|c|}
		\hline
		\mbox{Points} & p_1 & p_2 & p_3 & p_4 & p_5 \\
		\hline
		x & 0 & 1 & 2 & 3 & 4\\
		\hline
		y & 1 & 2 & 0 & 4 & 3\\
		\hline
		z & \textbf{3} & 1 & 4 & \textbf{2} & 0\\
		\hline
		\end{array}
	$
	\caption{Illustration of \textit{1D--move} with an initial (left) and a following (right) configuration}
	\label{mutation}
	\end{table*}

	\subsection{Performance Evaluation} \label{PerformanceEvaluation}
	
	\textit{1D--move} outperforms not only $m_{2}$ but $m_{3}$ as well. We reproduced the experiments made in~\cite{rimmel14} keeping the same value of parameter $p$ ($p=10$) for $4/25$, $9/10$ and $8/20$ to show its performance (Table~\ref{perf1dmove}). SA performs a linear thermal descent until temperature $T=0$ is reached. The initial temperature is set thanks to a series of preliminary runs.  We computed 100 effective runs and we present here the average within the 95\% confidence interval. 
	
	\begin{table}[h]
	\centering
	$\begin{array}{|c|c|c|c|}
	\hline
	\mbox{Instance} & m_2 & m_3 & \mbox{\textit{1D--move}} \\
	\hline
	4/25 & 177.59 \pm 0.29 & 177.67 \pm 0.29 & 180.51 \pm 0.27\\
	\hline
	9/10 & 156.24 \pm 0.10 & 156.06 \pm 0.08 & 156.54 \pm 0.06 \\
	\hline
	8/20 & 431.98 \pm 0.61 & 433.72 \pm 0.84 & 436.20 \pm 0.56\\
	\hline
	\end{array}$
	\caption{Performance of SA with different mutations}
	\label{perf1dmove}
	\end{table}
	
	To explain the efficiency of \textit{1D--move}, we can refer to SA on the Traveling Salesman Problem. Article \cite{tian99} shows that the perturbations which move the smallest number of edges are the best. 1D-move modifies the same number of points as $m_{2}$ and $m_{3}$. Consequently, $2(n-2)$ distances are modified by all these mutations. However, the changes on distances are smaller with \textit{1D--move} given that modifications on coordinates are $\pm 1$ thanks to the neighborhood property. We prove it with Eq.~(\ref{mutationChanges}). Our hypothesis is that this specific property explains why \textit{1D--move} is more efficient.
	
	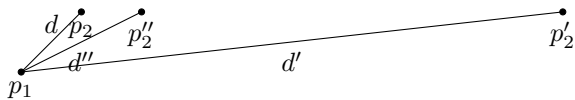
\begin{figure}[b]
\centering
\begin{tikzpicture}[scale=0.8]

\draw [black, fill=black] (0,0) circle (0.05);
\draw (0,0) node[below]{$p_1$};
\draw [black, fill=black] (1,1) circle (0.05);
\draw (1,1) node[below]{$p_2$};
\draw [black, fill=black] (2,1) circle (0.05);
\draw (2,1) node[below]{$p''_2$};
\draw [black, fill=black] (9,1) circle (0.05);
\draw (9,1) node[below]{$p'_2$};
\draw (0,0)--(1,1);
\draw (0.5,0.5) node[above]{$d$};
\draw (0,0)--(2,1);
\draw (1,0.5) node[below]{$d''$};
\draw (0,0)--(9,1);
\draw (4.5,0.5) node[below]{$d'$};

\end{tikzpicture}
\caption{Effect of $m_2$ and \textit{1D--move}}
\label{schemaDistance}
\end{figure}
	
	Let us consider a $n/k$ configuration $\omega$ and the two mutations $m_{2}$ and \textit{1D--move}. We want to prove that the changes due to these two mutations are not necessarily represented by the same order of magnitude. Let us note:
	\[
	 m_{2}: \omega \longrightarrow \omega'  ~\mbox{ and }~ \mbox{\textit{1D--move}}: \omega \longrightarrow \omega{''}.
	\]
	We assume the two mutations translate the point $p_{2}$ on a given dimension $j$. We also take a point $p_{1}$ of the configuration which remains invariant with these mutations ($p_{1}'$ and $p_{1}''$ equal to $p_{1}$). The objective is to find a configuration for which $\Delta d = \left|d_{p_{1}',p_{2}'} - d_{p_{1},p_{2}}\right|$ is equivalent to $n$ in the $m_{2}$ case. This difference is:
	\[
	\Delta d = \left|\sqrt{\sum_{l=1}^{k} (p_{2}'(l) - p_{1}(l))^{2}} - \sqrt{\sum_{l=1}^{k} (p_{2}(l)-p_{1}(l))^{2}}\right|.
	\]
	As most of the dimensions are not concerned by this move, we just note:
	\[
	\sum_{l=1,l\neq j}^{k} \left( p_{2}'(l) - p_{1}(l)\right)^{2} = \sum_{l=1,l\neq j}^{k} \left( p_{2}(l) - p_{1}(l)\right)^{2} = a^{2}.
	\]
	Variable $a$ depends on $n$ and $k$. If $p_{1}$ and $p_{2}$ are neighbors regarding all dimensions (except $j$), $a^{2} = k-1$. We take a configuration for which $a^{2} = k-1$, $p_{1}(j) = 0$, $p_{2}(j) = n-1$ and $p_{2}'(j) = 1$. This configuration is illustrated in Figure~\ref{schemaDistance}.
	
	By separating $j$ from other dimensions, we eventually find: 
	\begin{equation}
	\begin{array}{ll}
	\Delta d &= \left|\sqrt{a^{2} + (n-1)^{2}} - \sqrt{a^{2} + 1}\right| \\
	&= \left|\sqrt{k-1 + (n-1)^{2}} - \sqrt{k}\right|.
	\end{array}
	\label{mutationChanges}
	\end{equation}
	Some configurations respect the property: $k \ll n^{2}$, for which $\Delta d = \mathrm{O}(n)$. This means that the difference between two distances may take values with the order of magnitude $n$. It is not possible with \textit{1D--move}. Using the triangle inequality and the neighborhood property: $\Delta d \leq 1$.
	
	With $m_{2}$ (or $m_{3}$ which is more restrictive than $m_{2}$), $\Delta d$ sometimes reaches the order of magnitude $n$. We showed with $\Delta d \leq 1$ that it was impossible with \textit{1D--move} which allows the local search to be more regular.

\section{New evaluation function targeting Maximin}

	\subsection{Presentation of a Maximin effect: narrowing the distribution of distances} \label{cases}
	
We study the properties of distances obtained with the evaluation function $\phi_{p}$ in SA solutions. We represent all the distances of a configuration in histograms and identify properties that will allow us to establish a better evaluation function below. From now on, we distinguish three cases relative to values taken by $n$ and $k$. We note the mean of $D$ for any configuration as $\overline{D}(k,n) = \frac{k n (n+1)}{6}$ as shown in \cite{vandam09}.
\begin{figure}[H]
\centering{
			\includegraphics[width=0.7\columnwidth]{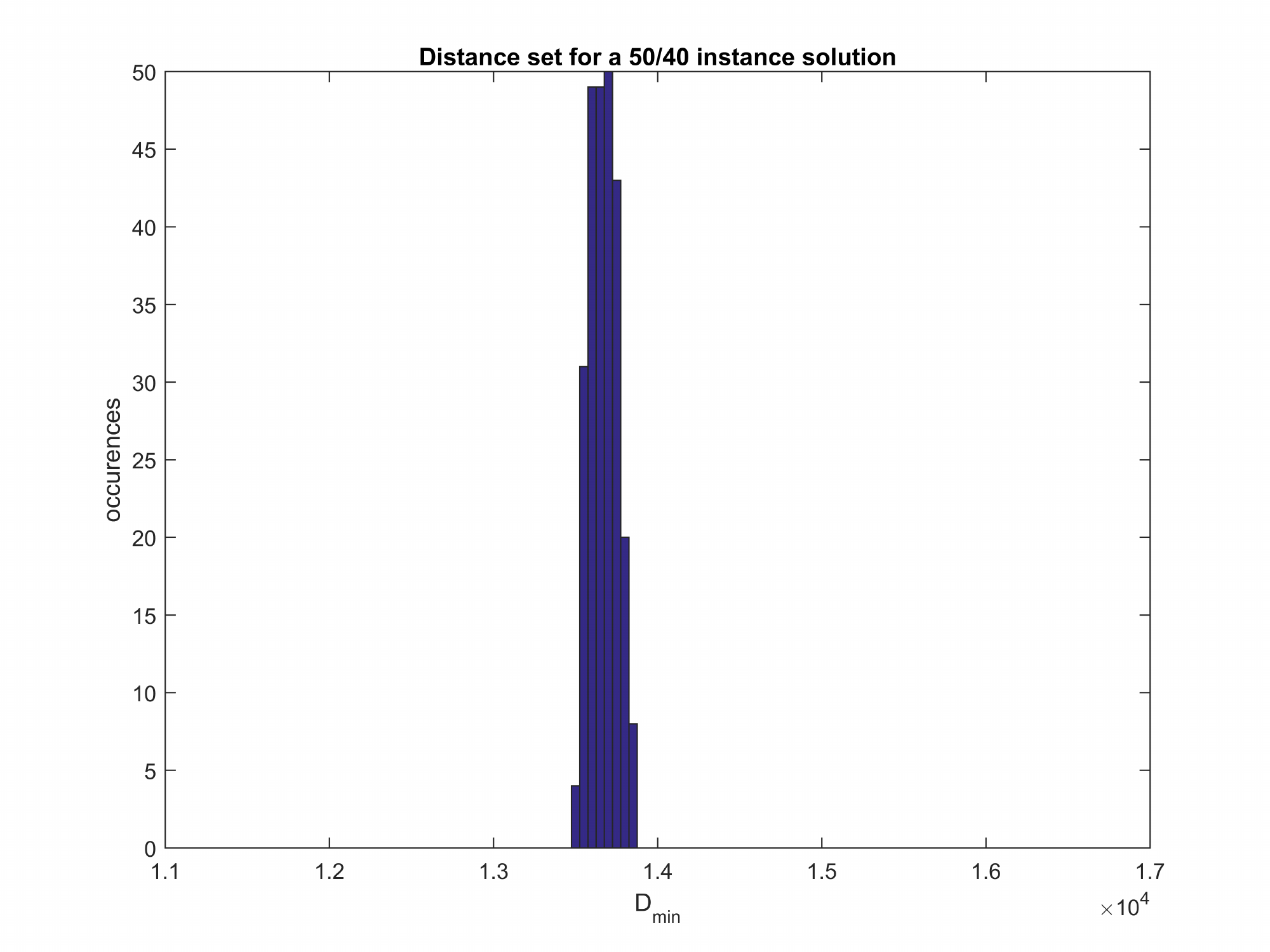}
			\includegraphics[width=0.7\columnwidth]{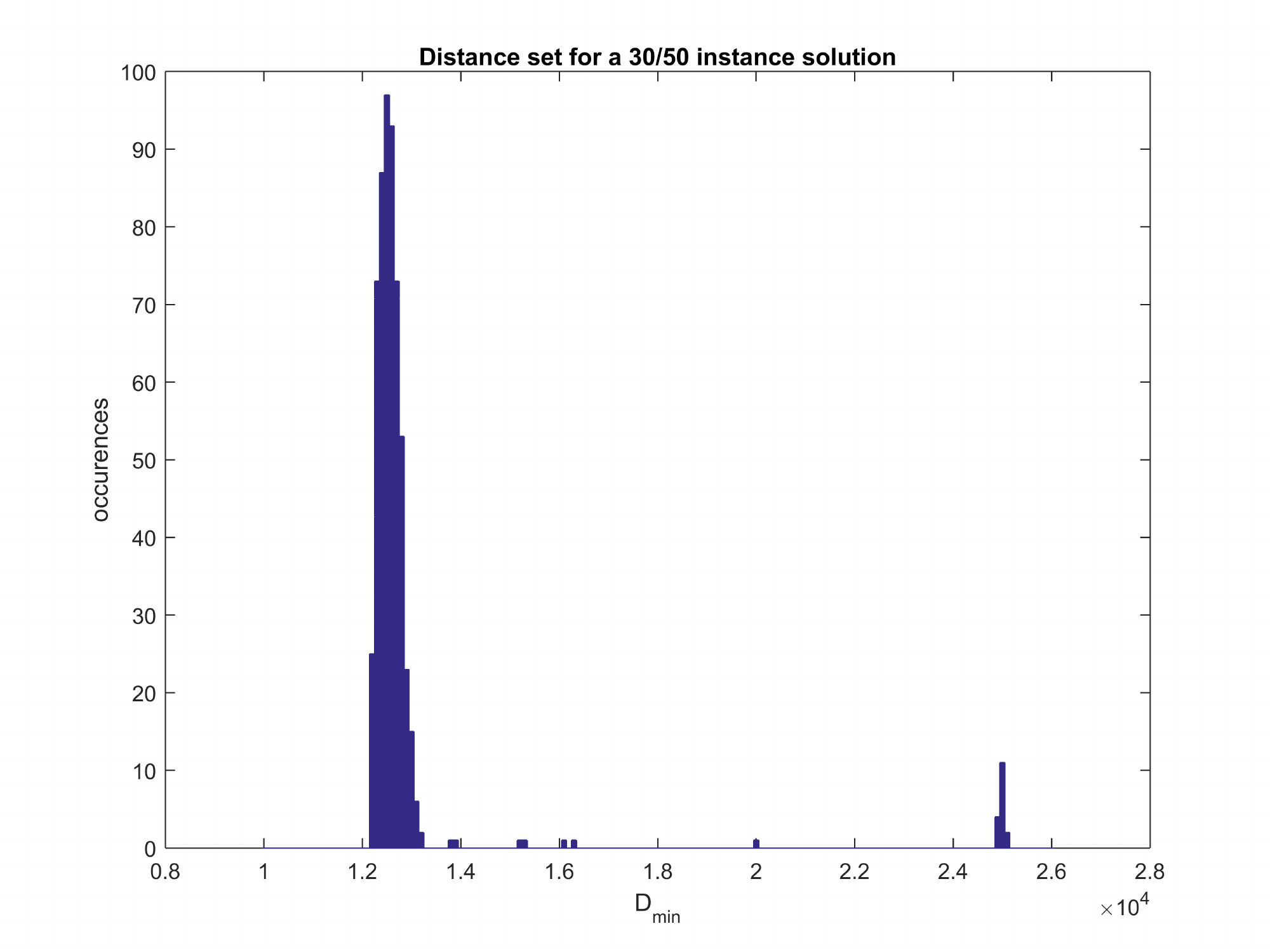}
			\includegraphics[width=0.7\columnwidth]{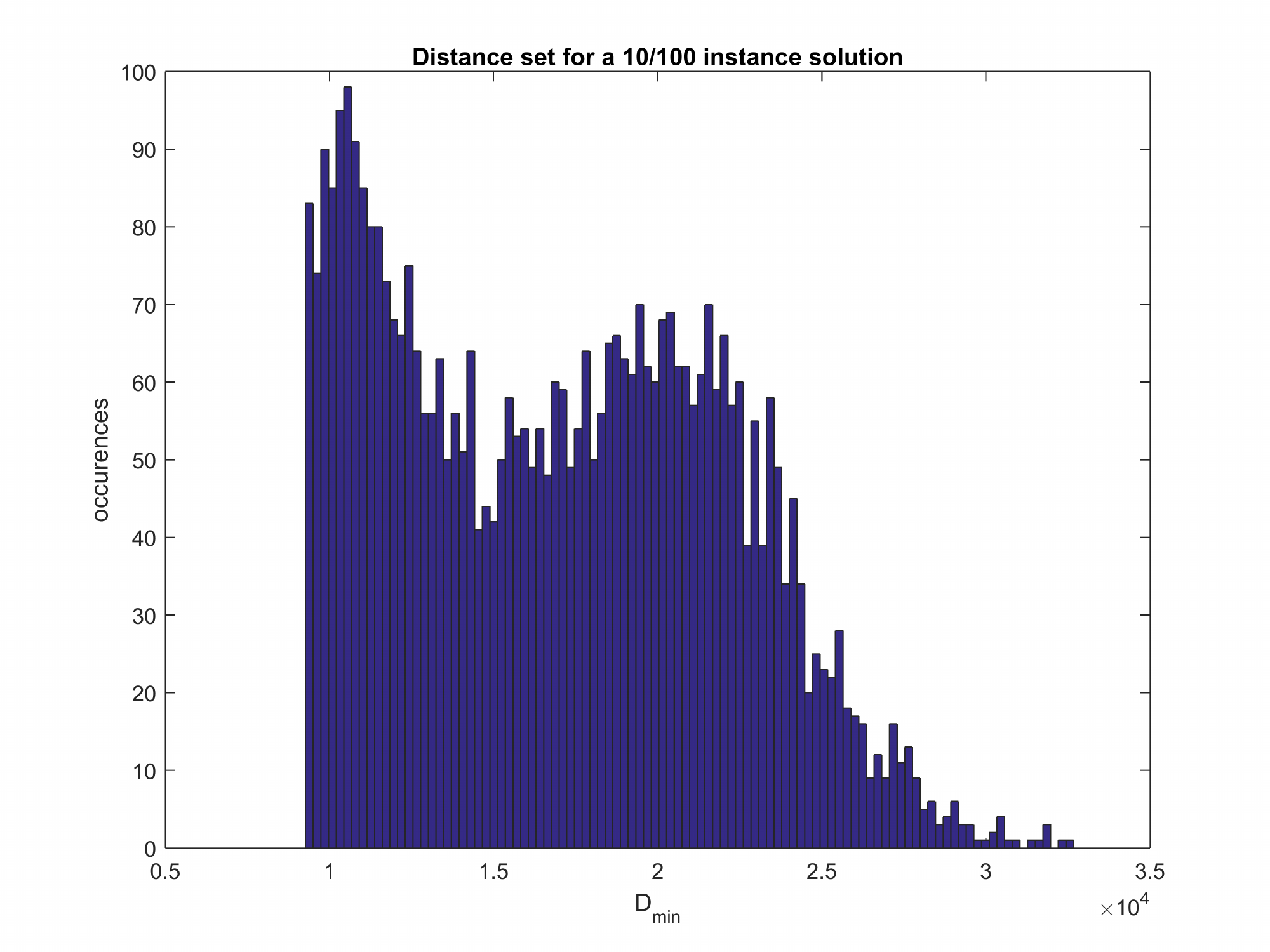}
}
	\caption{Histograms of distances for $50/40$, $30/50$ and $10/100$ solutions}
  \label{the3cases}
	\end{figure}
	
	\subsubsection*{\textbf{Case} $n \leq k$}
In this case (see Figure~\ref{the3cases}, $50/40$), the distances of potential solutions are concentrated around the mean. It is highly probable that two points at random taken will be neighbors. This explains why a point is close to all others in SA solutions and we talk about unimodal distribution. In our example $50/40$ in Figure~\ref{the3cases}, the statistical range of $D$ relative to $\overline{D}$, $\frac{D_{\mbox{\scriptsize{max}}} - D_{\mbox{\scriptsize{min}}}}{\overline{D}} = \frac{340}{13667} = 2.5\%$, in fact, is narrow. The rationale for this behavior is that when the number of points is less than the number of dimensions, it happens, in absence of constraints, that all the points are equidistant. Since the Latin constraint has to be respected, the points cannot be exactly equidistant. The distances, however, do not differ significantly.

	\subsubsection*{\textbf{Case} $k \leq n \leq 2k$}
	In this case (Figure~\ref{the3cases}, $30/50$), distributions are concentrated around two peaks. The first peak is mainly around the average distance (actually, there is a little shift between the peak and the mean because both the peaks preserve $\overline{D}$) and the second peak is located around the doubled average distance. Much more distances are concerned by the first peak. 
	
	We illustrate this phenomenon with the $30/50$ instance in Figure~\ref{the3cases}. We can explain this by the fact that it is possible for this many points to be placed in an hyperoctahedron. In such a geometric object, each point is at the same distance from every other point but one, which is farther away. Thus, the distribution of distances shows two values, with the smaller being represented much more frequently.
	
	In our example, $\overline{D}(30,50) = 12500$. Concerning the highest peak, the statistical range remains small compared with the mean: the ratio is $7.8\%$, larger than in the first case for the whole distribution. There are only seven distances located in the interval $\left[13183;24865\right]$
	
	\subsubsection*{\textbf{Case} $2k \leq n$} 
	In this last case (Figure~\ref{the3cases}, $10/100$), distances are distributed more uniformly. There is neither a dense peak nor a sparse interval. We observe a decrease of occurrences with an increase in the value of the distance.
	
	\subsubsection*{\textbf{Observations and consequences}}
	For the first case, the only peak is naturally thin thanks to SA and particularly $\phi_{p}$ action. There is a little point in trying to narrow it more. We note that for the two last cases ($k \leq n$), narrowing differences between distances lead to improve performance. We illustrate this on the $8/20$ instance. We represent distance sets of several possible solutions and observe that the best solutions have the most narrowed distributions. We compare two solutions in Figure~\ref{distribution_8_20_comp2} with $D_{\min} = 421$ and $D_{min} = 446$ which is the best solution found in \cite{rimmel14}. Indeed, we note that $D_{\min} = 446$ has the most narrow peak. We formulate the hypothesis that this property may be beneficial for SA performance. We introduce below a new evaluation function taking into account this aspect.
	
	\begin{figure*}[ht]
		\centering{\includegraphics[width=1\textwidth]{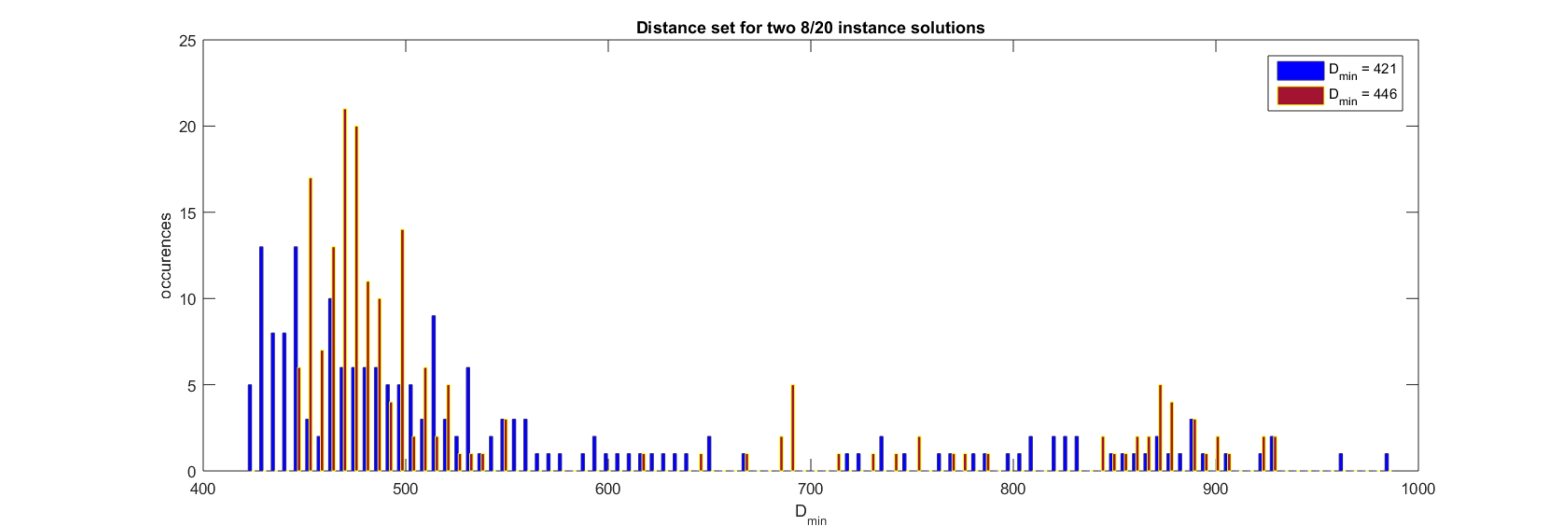}}
		\caption{Distance sets of two $8/20$ solutions}
		\label{distribution_8_20_comp2}
	\end{figure*}
	
	\subsection{Definition of evaluation function $\psi$}
	
	We propose an evaluation function $\psi_{p,\sigma}$ to replace the usual function $\phi_{p}$:
	\begin{equation}
	\psi_{p,\sigma} = \left( \sum_{i=1}^{\binom{n}{2}} w_{i} d_{i}^{-p} \right)^{\frac{1}{p}} \mbox{, where} ~w_{i} = \frac{1}{\sqrt{\sum_{j=1}^{\binom{n}{2}} e^{-\frac{\left|D_{j}-D_{i}\right|^{2}}{\sigma^{2}}}}}.
	\end{equation}
	
	 The idea is to add weights $w_{i} \geq 1$ for each distance term $d_{i}^{-p}$. These weights determine if the distance is close to other ones. If a distance is far from the others, the weight will be high. Consequently, it forces the distances to be close to each other.
	A single drawback of $\psi$ is its complexity in $\mathcal{O}(n^4)$. There are different ways to reduce this complexity. First, for instance, it is possible to consider only the differences which respect $\left|D_{j}-D_{i}\right|^2 \leq 5 \sigma^2$. In this way, we avoid the calculations of terms that may be considered as negligible ($e^{-5} \ll 1$). Instead of summing up $\binom{n}{2}$ distances, we can randomly choose $\mathcal{O}(n)$ distances~$D_{j}$.
	
		\begin{table*}[t]
	\centering
	$\begin{array}{|c|c||c|c|c|c|}
	\hline
	\mbox{Inst.} & \sigma & \phi_{10} ~\mbox{\&}~ m_2 & \psi_{10,\sigma} ~\mbox{\&}~ m_2 & \phi_{10} ~\mbox{\&}~ \mbox{\textit{1D-move}} & \psi_{10,\sigma} ~\mbox{\&}~ \mbox{\textit{1D-move}}\\
	\hline
	4/25 & 70 & 177.59 \pm 0.29 & 177.98 \pm 0.71 &  180.51 \pm 0.27 & 181.24 \pm 0.23\\
	\hline
	9/10 & 20 & 156.24 \pm 0.10 & 156.09 \pm 0.06 & 156.54 \pm 0.06 & 156.49 \pm 0.10 \\
	\hline
	8/20 & 65 & 431.98 \pm 0.61 & 433.58 \pm 0.70 & 436.20 \pm 0.56 & 445.28 \pm 0.45\\
	\hline
	\end{array}$
	\caption{Performance of SA with different setups for evaluation function and mutation}
	\label{perfpsi}
	\end{table*}
	
	\subsection{Tuning of parameter $\sigma$  and its justification}
	
	Let us focus on the parameter $\sigma$: given that we aim at furnishing a large number of scores, we need to tune it in a global way. It must depend directly on $n$ and $k$, without preliminary experiments for each instance $k/n$. Looking at the definition of $\psi_{p,\sigma}$, this variable is introduced in order to regulate the order of magnitude of the exponential term. We see that $\sigma$ should have approximately the same order of magnitude than the values taken by $\left|D_{j}-D_{i}\right|^{2}$. 
	
	This is why we try to give the expression of a linear function of $k$ and $n$ which is similar to typical values $\left|D_{j}-D_{i}\right|^{2}$. To establish it, we study the variance of a random variable: the tuning of $\sigma$ is founded on Theorem~\ref{th}.
	
	\begin{theorem}
	\label{th}
	Let $D(k,n)$ be the random variable representing any square distance in any configuration of instance $k/n$. We have $D\left(k,n\right) \sim \mathcal{N}\left(\frac{k n \left(n+1\right)}{6},g\left(n\right)\right)$ with $g\left(n\right) \sim \frac{7 k n^{4}}{180} + \mathcal{O}(n^3)$.
	\end{theorem}
	
	\begin{proof}
	Thanks to \cite{vandam09}, we know that $\mathbb{E}(D\left(k,n\right)) = \frac{k n (n+1)}{6}$. We note $\left(P_1,P_2\right)$ the random variable that gives any couple of points for $n/k$. The random variable $D\left(k,n\right)$ is a function of $\left(P_1,P_2\right)$. For any $1 \leq j \leq k$, we note $Y(j) = \left(P_1(j)-P_2(j)\right)^{2}$ and get $D\left(k,n\right) = \sum_{j=1}^{k} Y(j)$. As $Y(i)$ and $Y(j)$ are independent if $i\neq j$, we note $Y(i) = Y$ to keep the notation simple. If $k$ is high enough, we apply the Central Limit Theorem: $D\left(k,n\right) \sim \mathcal{N}\left(\frac{k n \left(n+1\right)}{6},k \mathbb{V}\mbox{ar}(Y) \right)$.
	We focus first on $\mathbb{E}(Y^{2}) = \mathbb{E}((P_1(j)-P_2(j))^{4})$:
\[ 
 	\mathbb{E}(Y^{2}) = \frac{\sum_{x=1}^{n} \sum_{y\neq x} (x-y)^{4}}{\frac{n (n-1)}{2}} = \frac{2\left( n\sum_{z=1}^{n-1} z^{4} - \sum_{z=1}^{n-1} z^{5} \right)}{n (n-1)} = \frac{n^{4}}{15} + \mathcal{O}\left(n^{3}\right). 
 	\]

	We thus deduce $\mathbb{V}\mbox{ar}(Y) = \mathbb{E}(Y^{2}) - \mathbb{E}(Y)^{2} = \frac{n^{4}}{15} - \frac{n^{4}}{36} + \mathcal{O}\left(n^{
	3}\right) \sim \frac{7 n^{4}}{180}$.
	\end{proof}

	\begin{table*}[t]
		\centering
		\small{
		$\begin{array}{|c|c|c|c|c|c|c|c|c|c|c|}
		\hline
		\multicolumn{2}{|c||}{\hspace{1mm} $\backslashbox{$n$}{$k$}$ \hspace{1mm}} & \hspace{3mm} 3 \hspace{3mm} & \hspace{3mm} 4  \hspace{3mm} & \hspace{3mm} 5 \hspace{3mm} & \hspace{3mm} 6 \hspace{3mm} & \hspace{3mm} 7 \hspace{3mm} & \hspace{3mm} 8 \hspace{3mm} & \hspace{3mm} 9  \hspace{3mm} & \hspace{3mm} 10 \hspace{3mm} \\
		\hline
		\hline
		\multicolumn{2}{|c||}{3} & 6 & 7 & 8 & 12 & 13 & 14 & 18 & 19 \\
		\hline
		\multicolumn{2}{|c||}{4} & 6 & 12 & 14 & 20 & 21 & 26 & 28 & 33 \\
		\hline
		\multicolumn{2}{|c||}{5} & 11 & 15 & 24 & 27 & 32 & 40 & 43 & 50 \\
		\hline
		\multicolumn{2}{|c||}{6} & 14 & 22 & 32 & 40 & 47 & 54 & 62 & 68 \\
		\hline
		\multicolumn{2}{|c||}{7} & 17 & 28 & 40 & 52 & 62 & 72 & 81 & 91 \\
		\hline
		\multicolumn{2}{|c||}{8} & 21 & 42 & 50 & 66 & 80 & 91 & 103 & 116 \\
		\hline 
		\multicolumn{2}{|c||}{9} & 22 & 42 & 61 & 82 & 95 & 114 & 128 & 144 \\
		\hline
		\multicolumn{2}{|c||}{10} & 27 & 50 & 82 & 95 & 113 & \textbf{134} & 158 & 175 \\
		\hline
		\multicolumn{2}{|c||}{11} & 30 & 55 & 82 & 111 & \textbf{133} & 157 & 184 & 211 \\
		\hline
		\multicolumn{2}{|c||}{12} & 36 & 63 & 94 & 142 & 158 & \textbf{184} & 213 & 243 \\
		\hline
		\multicolumn{2}{|c||}{13} & 41 & 70 & \textbf{107} & 143 & 184 & \textbf{214} & 246 & 279 \\
		\hline
		\multicolumn{2}{|c||}{14} & 42 & 78 & \textbf{109} & \textbf{162} & 220 & \textbf{247} & 282 & 318 \\
		\hline
		\multicolumn{2}{|c||}{15} & 48 & 89 & \textbf{135} & \textbf{179} & 228 & \textbf{281} & \textbf{323} & 363 \\
		\hline
		\multicolumn{2}{|c||}{16} & 50 & 94 & 154 & \textbf{200} & \textbf{254} & \textbf{328} & 364 & \textbf{412} \\
		\hline
		\multicolumn{2}{|c||}{17} & 56 & 102 & 163 & 221 & \textit{277} & 343 & 413 & 462 \\
		\hline
		\multicolumn{2}{|c||}{18} & 57 & 114 & \textbf{176} & \textbf{249} & \textbf{306} & \textbf{376} & 469 & 515 \\
		\hline
		\multicolumn{2}{|c||}{19} & 62 & \textbf{123} & \textbf{193} & \textbf{268} & \textbf{336} & 408 & 491 & 576 \\
		\hline
		\multicolumn{2}{|c||}{20} & 66 & \textbf{138} & \textbf{210} & \textbf{293} & \textbf{372} & \textbf{448} & 528 & 645 \\
		\hline
		\multicolumn{2}{|c||}{21} & 69 & 149 & 232 & \textbf{315} & \textbf{401} & 482 & 570 & \textbf{674} \\
		\hline
		\multicolumn{2}{|c||}{22} & 82 & \textbf{154} & \textbf{246} & 347 & 433 & 525 & 623 & \textbf{721} \\
		\hline
		\multicolumn{2}{|c||}{23} & 82 & \textbf{165} & 260 & 364 & \textbf{468} & 566 & 667 & \textit{773} \\
		\hline
		\multicolumn{2}{|c||}{24} & 83 & \textbf{173} & \textbf{276} & \textbf{391} & 506 & 609 & \textbf{720} & 837 \\
		\hline
		\multicolumn{2}{|c||}{25} & 89 & 183 & 294 & 419 & \textbf{541} & 657 & \textit{768} & \textbf{897} \\
		\hline
		\end{array}$
		}
		\caption{Highscores obtained with ``all purpose'' tuning}
		\label{highscores}
	\end{table*}

	We propose a global tuning of $\sigma^{2}$ as a linear function of the variance of our configurations. As computing the variance of a configuration, at every iteration, would be expensive, we formulate the hypothesis that the variance of the square distances set of the SA solutions follows the function $g\left(n\right)$ above. The idea of the tuning is to consider $\sigma^{2}$ linearly dependent on the variance of the random variable $D\left(k,n\right)$. In the weights $w_{i}$, we compare the difference between the current distance and an extra one with $\sigma$ by calculating $\frac{\left|D_{j}-D_{i}\right|^{2}}{\sigma^{2}}$ in order to identify which differences $D_{j}-D_{i}$ have to be taken into account. 
	
	In the case $n \geq 2k$, we assume $\sigma^{2} = c k n^{4}$. According to several experiments series, we identify a good compromise with $c = \frac{1}{300}$. 
	
	In the case $n \leq k$, leading to unimodal distributions, $\psi$ does not bring more interesting results than $\phi$. It is equivalent to assuming $c$ to be very large ($c \rightarrow \infty$). 
	
	Finally, the case $k \leq n \leq 2 k$ which is an intermediary of the two previous cases, can be tuned with $\sigma = 2 c k n^4$. This proposition does not obviously represent the best tuning for all possible instances but gives an efficient and simple solution for the tuning of $\sigma$. 
	
	It is necessary to mention that the case $k \leq n \leq 2 k$ is the case where tuning is essential: to be as efficient as possible, the value of $\sigma$ has to be carefully selected. Table~\ref{perfpsi} shows the impact of $\psi_{p,\sigma}$ on SA performance with mutations $m_2$ and \textit{1D-move}. We keep the same experimental setup as in Subsection~\ref{PerformanceEvaluation}: SA makes a thermal linear descent, the results presented come out from 100 runs and the average is within the 95\% confidence interval.
	
	In Table~\ref{highscores}, we update scores for the same instances as in \cite{rimmel14}. The results were produced with $10^{7}$ iterations and $p=5$. Our function $\psi_{p,\sigma}$ is used when $k \leq n$, $\phi_p$ elsewhere. We note in bold type improved results and in italics results worse than \cite{rimmel14}. For $4 \leq k \leq 8$, the use of \textit{1D--move} and $\psi_{p,\sigma}$ allows us to exceed a large number of scores but this improvement is less significant for other values. For $k = 3$, we suppose that the new tools are not able to outperform previous results because the results are already optimal or very good. For $k = \left\{9,10\right\}$, a credible hypothesis is that the value of $\frac{n}{k}$ is so close to 1 that the effect of $\psi_{p,\sigma}$ is weak. Generally, results could be better with a specifically adapted tuning. Here, we established temperature, $p$ and $\sigma$ by making compromises between all the instances. However, in a real life case, by treating complex systems, we work on a defined instance with $k$ and $n$ fixed. In such circumstances, we naturally advice to customize the tuning of the different parameters by making preliminary experiments on this very instance. We expect that such an approach would produce results outperforming those in Table~\ref{highscores}.
	
	\section{Conclusion}
	
In this article, we introduce new techniques to treat the Maximin LHD construction. The first one is the \textit{1D--move} mutation especially dedicated to the LHD structure. It is very efficient for a local search on LHDs because it makes it possible to follow a step-by-step path on the cost surface without jumping over possible minima. The second tool, the evaluation function $\psi_{p,\sigma}$ directly focuses upon Maximin optimization. 

As numerous problems, among them Maximin Designs, involve this criterion, we emphasize that this function can be used for many other applications. In the Maximin LHD context, the function $\psi_{p,\sigma}$ tries to find solutions by narrowing a set of possible distances. SA, with \textit{1D--move} and $\psi_{p,\sigma}$, gave results better than those considered to be ``the best known'' for the majority of cases without any dedicated tuning.

\bibliographystyle{abbrv}
\bibliography{ArticleLH-CoRR}

\end{document}